\documentclass[letterpaper]{article} % DO NOT CHANGE THIS
\usepackage{aaai2026}  % DO NOT CHANGE THIS
\usepackage{times}  % DO NOT CHANGE THIS
\usepackage{helvet}  % DO NOT CHANGE THIS
\usepackage{courier}  % DO NOT CHANGE THIS
\usepackage[hyphens]{url}  % DO NOT CHANGE THIS
\usepackage{graphicx} % DO NOT CHANGE THIS
\usepackage{natbib}  % DO NOT CHANGE THIS AND DO NOT ADD ANY OPTIONS TO IT
\usepackage{caption} % DO NOT CHANGE THIS AND DO NOT ADD ANY OPTIONS TO IT
\usepackage{algorithm}
\usepackage{algorithmic}
\usepackage{amsmath}
\usepackage{amsfonts}
\usepackage{booktabs}
\usepackage{multirow}
\usepackage{tabularx}
\usepackage{makecell}
\usepackage{amsthm}
\usepackage{amssymb}

\newtheorem{prop}{Proposition}
\usepackage{newfloat}
\usepackage{listings}
\DeclareCaptionStyle{ruled}{labelfont=normalfont,labelsep=colon,strut=off} % DO NOT CHANGE THIS
\floatstyle{ruled}
\newfloat{listing}{tb}{lst}{}
\floatname{listing}{Listing}
\title{Probabilities Are All You Need: A Probability-Only Approach to \\ Uncertainty Estimation in Large Language Models}

\author{
    Manh Nguyen, Sunil Gupta, Hung Le
}
\affiliations{
    Applied Artificial Intelligence Initiative\\
    Deakin University\\
    Geelong, Victoria, Australia\\
    \texttt{\{manh.nguyen, sunil.gupta, thai.le\}@deakin.edu.au}
}

\usepackage{bibentry}
\begin{document}

\maketitle

\begin{abstract}
Large Language Models (LLMs) exhibit strong performance across various natural language processing (NLP) tasks but remain vulnerable to hallucinations, generating factually incorrect or misleading outputs. Uncertainty estimation, often using predictive entropy estimation, is key to addressing this issue. However, existing methods often require multiple samples or extra computation to assess semantic entropy. This paper proposes an efficient, training-free uncertainty estimation method that approximates predictive entropy using the responses' top-$K$ probabilities. Moreover, we employ an adaptive mechanism to determine $K$ to enhance flexibility and filter out low-confidence probabilities. Experimental results on three free-form question-answering datasets across several LLMs demonstrate that our method outperforms expensive state-of-the-art baselines, contributing to the broader goal of enhancing LLM trustworthiness.
\end{abstract}

\begin{links}
    \link{Code}{https://github.com/manhitv/PRO}
    % \link{Datasets}{https://aaai.org/example/datasets}
    % \link{Extended version}{https://aaai.org/example/extended-version}
\end{links}

\section{Introduction}

Large language models (LLMs) \cite{achiam2023gpt, touvron2023llama, team2023gemini} have demonstrated remarkable capabilities in solving NLP tasks. However, they still suffer from hallucinations \cite{xiao2021hallucination, malinin2020uncertainty}, where LLMs generate outputs that are factually incorrect or misaligned with the provided context, even when given valid inputs. Therefore, detecting and mitigating hallucinations remains a significant challenge in the development of trustworthy AI systems.

Estimating the uncertainty of model generations is a promising direction to improve the reliability of LLMs \cite{manakul2023selfcheckgpt}. By quantifying the confidence of their outputs, uncertainty measures help identify cases where a model is likely to produce incorrect or misleading information. Existing uncertainty estimation has been widely explored and can be categorized into two main types: training-based and training-free methods. Training-based methods involve training a custom model or leveraging a pretrained model to assess the truthfulness of the output \cite{azaria2023internal, kuhn2023semantic, farquhar2024detecting, nikitin2024kernel, qiu2024semantic}. These methods often require access to low-level features of the LLMs, which may not always be available in practice. In contrast, training-free methods are simpler, as they rely only on token probabilities from the original LLM to estimate predictive entropy and log probability without requiring additional model training \cite{huang2025survey}. In this work, we focus on training-free approaches due to their simplicity and practical usability.

Early training-free approaches use log-probability to measure the model's uncertainty. \citet{manakul2023selfcheckgpt} propose several metrics based on token log-probabilities, including taking the average and max over them as baselines for uncertainty estimation. More advanced approaches use predictive entropy \cite{lindley1956measure} to characterize the total uncertainty in a model's output distribution. Predictive entropy is reliable because it captures the overall uncertainty in a model’s output, considering the entire probability distribution rather than just a single prediction. This measure has demonstrated robustness across diverse tasks, such as data-to-text generation \cite{xiao2021hallucination}, neural machine translation \cite{malinin2020uncertainty}, and abstractive summarization \cite{van2022mutual}. More recently, predictive entropy has been extended to question-answering tasks in the context of LLMs \cite{kuhn2023semantic, lin2023generating, duan2023shifting}. However, it is impossible to compute predictive entropy exactly due to the high sampling cost associated with generating all possible outputs. Several studies have attempted to refine entropy-based uncertainty estimation by incorporating token- or sentence-level weighting. For example, \citet{duan2023shifting} introduce token and sentence weights based on their importance within specific spans of text, such as a set of tokens within a sentence or a group of responses, and combine these weights into an improved version of predictive entropy. \citet{kuhn2023semantic} and \citet{farquhar2024detecting} propose semantic entropy (SE), which quantifies uncertainty by computing the predictive entropy of semantically related response clusters. This approach has opened new possibilities for uncertainty estimation by adjusting for semantic similarities. While these semantic methods have shown strong empirical performance, due to the inherent complexity of semantic meaning, accurately clustering responses remains challenging, particularly since a single sentence can belong to multiple clusters. Additionally, these methods frequently depend on external models, such as natural language inference (NLI) models \cite{he2020deberta}, to assess semantic distances between responses, yet the relationships between clusters in terms of semantic similarity remain unclear.

Recently, \citet{aichberger2024rethinking} introduce G-NLL, which estimates uncertainty using only the probability of the most likely generation  (i.e., the most likely complete sentence). This suggests that reliable estimation is possible without modeling response semantics. However, we argue that relying solely on the most likely generation fails to accurately approximate the uncertainty, resulting in suboptimal performance in certain cases.
% This opens a new approach for accurate uncertainty assessment, as the inference costs of LLMs can be considerable. Inspired by these findings, we aim to propose a generalized version of G-NLL that utilizes the probabilities of the most probable generations.
In this work, we propose a novel training-free uncertainty estimation method, dubbed \textbf{PRO} (\textbf{PR}obability \textbf{O}nly), that is based entirely on the top \(K\) generations with the highest probabilities. Unlike existing semantics-based methods, our approach does not require additional computations involving semantic meaning or response embeddings, making it both simple and cost-effective. 
% We assume that LLMs already integrate semantics implicitly when assigning probabilities to different outputs, meaning explicit clustering of semantically similar responses might be unnecessary. -- not sure
We formulate uncertainty estimation as a predictive entropy (PE) approximation problem and theoretically demonstrate that PE can be estimated using the highest probabilities from sampled generations. Each generation's probability is computed using the Negative Log-Likelihood (NLL). The resulting uncertainty is further refined by applying a probability threshold, which filters out low-probability generations to ensure that only high-confidence responses contribute to the uncertainty measure, ultimately yielding the final top-$K$ probabilities.
We hypothesize that while a larger $K$ may improve entropy estimation, incorporating low-probability responses adds noise and contributes negligibly to the final entropy value.
% The adaptive constraint helps maintain computational efficiency while preserving the reliability of the uncertainty approximation. 
% The final uncertainty approximation is more efficient and adaptable, particularly when the response distribution varies across different tasks. 
In summary, our main contributions are as follows: 
\begin{itemize}
    \item We propose a simple, theoretically motivated entropy approximation for uncertainty estimation that uses only the top probabilities from sampled generations.
    \item We introduce a probability threshold as a hyperparameter to adaptively quantify uncertainty, enhancing both flexibility and empirical performance.
    \item We conduct extensive experiments to show that our method outperforms existing baselines for uncertainty estimation across various free-form question-answering tasks and LLMs.
\end{itemize}

\section{Related Work}

% This paper belongs to the research domain of uncertainty quantification and estimation in LLMs, which aims to develop quantitative indicators for assessing the trustworthiness of LLM-generated responses. While the primary focus of the LLM community remains on improving model efficiency, there is a growing body of research dedicated to evaluating uncertainty in LLMs.

%One direction is directly ask LLM to generate uncertainty measure. [Ref, verbalize consistency, Tian et al., 2023] quantify uncertainty by parsing the verbalized confidence expression (e.g.,\textit{"very confident"}, \textit{"100\%"}, etc.) and mapping it to a numerical value. [SC, Wang et al., 2023] propose self-consistency to assess the uncertainty of a generation by measuring consistency of randomly sampled responses. One of the main advantage of these methods is that they do not requires access to the logits or hidden states and could be applied to black box LLMs.

In LLM's uncertainty estimation literature, one common direction is to focus on analyzing the internal states of the LLMs and training external models to predict confidence and uncertainty. For example, \citet{azaria2023internal} train a model to assess the truthfulness of the output based on specific hidden layers. \citet{chuang2024lookback} instead rely on attention layers to compute a lookback ratio, which measures the model's focus on the provided context, arguing that it correlates positively with the correctness of the response. Additionally, \citet{chen2024inside} estimate uncertainty by analyzing eigenvalues of the response covariance matrix, capturing the semantic consistency in the dense embedding space. In contrast, our method relies solely on token probabilities rather than internal states, making it significantly simpler and more efficient. This eliminates the need for additional model training or complex computations on hidden layers, reducing computational costs and enabling broader applicability across different models.

Another prevalent approach is using logits or probability scores to estimate uncertainty. Predictive entropy is widely used for this purpose \cite{malinin2020uncertainty} as it quantifies the total uncertainty in a model's output distribution. Several studies have explored enhancing predictive entropy by incorporating weighting functions. \citet{duan2023shifting} and \citet{bakman2024mars} utilize a transformer-based model to compute weights at token, phrase, and sentence levels, while \citet{zhang2023enhancing} leverage named entity recognition (NER) models to detect and assign greater weights to key terms. These approaches are simple yet only show moderate performance in practice. 
%Similarly, \cite{bakman2024mars} trained a custom model to compute weights at the phrase level. Another approach, \cite{zhang2023enhancing}, proposed calibrating predictive entropy by using a named entity recognition (NER) model \cite{honnibal2017spacy} that assigns greater weights to key terms. 
In recent years, semantic-based uncertainty estimation has gained significant attention. \citet{kuhn2023semantic} and \citet{farquhar2024detecting} propose semantic entropy (SE), which quantifies uncertainty by considering semantic relationships between generated responses. Inspired by SE, \citet{lin2023generating} propose uncertainty metrics by extracting semantic matrices, which are constructed by analyzing the relationships among generated responses. Semantic entropy has been further extended in several ways. For example, \citet{nikitin2024kernel} apply kernel functions applied to a semantic graph, resulting in a generalized version of SE. \citet{qiu2024semantic} advance this approach by introducing semantic density (SD), which is computed from a generalized semantic space and directly measures uncertainty using kernel functions without the need for clustering. To the best of our knowledge, this method represents the state-of-the-art technique in the field of uncertainty estimation. Although promising, these methods typically incur high computational costs and depend on pretrained models to assess semantic distances. In contrast, our method introduces minimal computational cost by relying solely on the token probability provided by the language models, without the need for additional inference or semantic similarity calculations.
%A generalized version of SE is proposed by \cite{nikitin2024kernel}, which applies kernel functions over a semantic graph, where the edges represent similarity scores between responses in terms of meaning. More recently, \cite{qiu2024semantic} extend this direction, introducing a generalized semantic space and measuring uncertainty directly through kernel functions without the need for clustering. To the best of our knowledge, this method represents state-of-the-art technique in the field of uncertainty estimation. 

Recently, some studies have also used the likelihood of a single output sequence as a simple yet effective uncertainty estimation \cite{aichberger2024rethinking,vazhentsev2024unconditional}. Our work also aims to quantify uncertainty through the most probable generations. However, unlike prior approaches that rely on a single or a fixed number of responses, we introduce a thresholding mechanism for selecting top-$K$ generations adaptively. More importantly, we derive a general formulation for entropy approximation, offering a more flexible and theoretically grounded measure of uncertainty.

\section{Preliminaries}

\paragraph{Problem Statement}

Information theory \cite{lindley1956measure} provides a systematic approach to measure uncertainty by defining it as the entropy of the output distribution:

\begin{equation}
  \label{eq:pe_origin}
  PE(x) = H(Y|x) = -\int p(y|x)\log p(y|x)dy
\end{equation}
where \(Y\) is the output random variable, \(x\) is the input, and \(H(Y |x)\) is a conditional entropy which represents average uncertainty about \(Y\) when \(x\) is given. A low predictive entropy indicates a heavily concentrated output distribution, whereas a high one indicates that many possible outputs are similarly likely.

In the context of LLMs, we can measure the uncertainty of a generation as:
\begin{equation}
  \label{eq:pe_llm}
  U(x) = H(Y|x) = -\sum_{y} p(y|x)\text{log}p(y|x)
\end{equation}
where, in practice, $y\in \{y_1, y_2, ..., y_N\}$ is the finite set of generated sequences (i.e., the samples of $Y$) given prompt \(x\) using a specific LLM. 

\paragraph{Generation Probability} The probability of generating sequence \(y\) given a prompt \(x\) is typically factorized as the product of conditional probabilities over its individual tokens:

\begin{equation}
    \label{eq:token_product}
    p(y|x) = \prod_{t=1}^{T} p(y^t|y^{<t}, x)
\end{equation}
where \(T\) is the length of the generated sequence, and \(y^t\) is the token at position \(t\). Taking the logarithm, we get:

\begin{equation}
    \label{eq:token_log}
    \log p(y|x) = \sum_{t=1}^{T}\log p(y^t|y^{<t}, x).
\end{equation}

This quantity can be computed empirically using the \textbf{Negative Log-Likelihood (NLL)} \cite{aichberger2024rethinking}, which is commonly used as a loss function in training LLMs:

\begin{equation}
    \label{eq:nll}
    \text{NLL}(y|x) = -\sum_{t=1}^{T} \log p(y^t|y^{<t}, x).
\end{equation}
% Alternatively, a normalized version known as the \textbf{Average Log-Likelihood (ALL)} \cite{manakul2023selfcheckgpt} can be computed by dividing by the sequence length:
% \begin{equation}
%     \label{eq:all}
%     \text{ALL}(y|x) = - \dfrac{1}{T} \sum_{t=1}^{T}\log p(y_t|y_{<t}, x).
% \end{equation}
% This formulation accounts for varying sequence lengths and ensures comparability across different generations. 

The token-level probabilities \(p(y^t|y^{<t}, x)\) are typically derived from the model's output logits using the softmax function, making this approach a standard method for evaluating the likelihood of generated text in LLMs. 

\section{Methodology}

\subsection{Uncertainty Approximation}
Calculating the uncertainty directly from Eq. \eqref{eq:pe_llm} is computationally expensive, as generating all possible responses \(Y\) for an accurate estimate would require significant resources. A typical way to estimate the predictive entropy is sampling \(N\) generations given prompt \(x\). Based on these samples, we propose selecting top-\(K\) ($K<N$) generations (\(y^*_1, y^*_2,..., y^*_K\)) with the highest probabilities to approximate uncertainty, i.e.:

\begin{align}
    \label{eq:topk_condition}
    p(y^*_1|x) &\geq p(y^*_2|x) \geq \dots \geq p(y^*_K|x), \\ 
    p(y^*_K|x) &\geq p(y_i|x), \forall i \in \{K+1, \dots, N\}.
\end{align}

For simplicity, we use 
% \(p(y_j)\) to denote probability of a specific generation \(p(y_j|x)\), and 
\(p^*_i\) to represent the probability of the top \(i\)-th generation \(p(y^*_i|x)\). 
Given these notations, we introduce an approximation of PE as a \textbf{PR}obability-\textbf{O}nly uncertainty score (PRO):
\begin{equation}
        \label{eq:approx_final}  
        PRO(x) = -\log p^*_K - \sum_{i=1}^{K} p^*_i \log \frac{p^*_i}{p^*_K}
\end{equation}
Here, using Eq. \eqref{eq:token_log} and the standard NLL from Eq. \eqref{eq:nll}, we can write the probability of a particular generation as follows:

\begin{equation}
    \label{eq:p_nll}
    p(y|x) = e^{-\text{NLL}(y|x)}
\end{equation}

We theoretically establish that PRO serves as a lower bound for predictive entropy, as formalized in the following proposition:

\begin{prop}
    \label{prop:prop1}
    Let \( \mathbf{y}^* = (y^*_1, y^*_2, \dots, y^*_K) \) be the top \( K \) generations of a LLM given prompt \( x \). The predictive entropy approximation using the top $K$ probabilities satisfies the following inequality:
    \begin{equation}
        \label{eq:approx_final_prop}  
        H(Y|x) \geq -\log p^*_K - \sum_{i=1}^{K} p^*_i \log \frac{p^*_i}{p^*_K}
    \end{equation}
\end{prop}
\begin{proof}
See Appendix~\ref{sec:prop_proof}. 
\end{proof}

While this lower bound is computationally efficient, it still captures the essential characteristics of uncertainty. Furthermore, estimating uncertainty with a lower bound helps avoid being overly influenced by outlier responses, i.e., low-probability generations that contribute disproportionately to the entropy due to their high diversity. These rare responses may not meaningfully represent the model’s general behavior but can artificially inflate the uncertainty estimate. By concentrating on the top $K$ most likely responses, our method reduces sensitivity to such noise and provides a more stable approximation. In this sense, the lower bound offers a more concentrated view of uncertainty, emphasizing the model’s confidence in its most plausible outputs. Notably, when our approximation is high, it still indicates that the true predictive entropy \(H(Y\mid x)\) must also be high, indicating more uncertainty in the model's output. This is important for detecting highly uncertain cases, where the model is unsure and generates diverse responses.
% Our approximation, by narrowing the focus to high probable responses, captures the most confident variations, leading to better identification of high-uncertainty situations. 
% This makes the uncertainty estimation more sensitive to cases with ambiguous or uncertain input, which is critical in many real-world applications where detecting uncertainty is crucial.

When \(K=1\), the right-hand side of Eq. \eqref{eq:approx_final_prop} becomes the negative log-probability of the most likely generation, as in \citet{aichberger2024rethinking}. Although using the most likely generation demonstrates robustness in performance, it is not always sufficient to capture the distribution of possible responses, especially in cases of high uncertainty or ambiguous prompts. Relying only on the top-1 response may overlook other plausible responses, leading to a less accurate uncertainty estimation. Our method accounts for a broader set of likely responses, yielding a more representative and stable estimate of uncertainty. This benefit is empirically demonstrated in Section~\ref{subsec:multiple-k}, where using larger $K$ improves uncertainty estimation and leads to better performance on downstream tasks.

\subsection{On Top-$K$ Selection}

% Our approximation, by focusing on the top \(K\) highest-probability generations, gives a tighter estimate because we prioritize higher-confidence responses and discard less informative, low-probability ones. This tighter bound leads to a more accurate and efficient uncertainty measure, especially when dealing with responses that exhibit significant variance or when the response space is large. 
To select $K$, one could simply set a fixed number of $K<N$. However, we argue that a fixed $K$ is suboptimal because the probability distributions of LLM generations may include low-probability responses that are less reliable for uncertainty estimation. These low-probability responses can introduce noise, reducing the effectiveness of the uncertainty measure. To address this, we introduce an adaptive constraint that filters out low-probability responses that would otherwise dilute the quality of the approximation. This constraint is essential for ensuring that the uncertainty measure reflects only the most confident responses while avoiding noise from less probable, potentially irrelevant generations. 
% By applying this constraint, we make sure that the responses with the highest confidence are the ones influencing the final uncertainty estimate, thus maintaining the integrity of the approximation.
%Eq. \eqref{eq:approx_final} offers a flexible approach to approximating the final uncertainty by allowing selective choice of \(K\). While \(K\) can be pre-set for a specific model and dataset, the probability distribution of responses can vary significantly across different prompt-model pairs. To address this variability, we introduce a threshold \(\alpha\) to dynamically determine \(K\), ensuring that only sufficiently probable generations are considered. This is formulated as follows:
The adaptive constraint is defined as follows:

\begin{equation}
    \label{eq:select_k}
    \mathbf{p}_K = \{ p_k \mid \ p_k \geq \alpha, 1 \leq k \leq N\},
\end{equation}
where \(\mathbf{p}_K\) is the selected set of top \(K\) probabilities, and \(\alpha\) ($0 \leq \alpha \leq p^*_1\leq 1$) is a tunable hyperparameter that truncates the output probability distribution \(p_k\). Larger \(\alpha\) entails more aggressive truncation, keeping only high-probability generations, whereas smaller \(\alpha\) allows generations with lower probabilities to be included. When \(\alpha=0\), we use all \(N\) generations, as no probability threshold is applied. On the other hand, when \(\alpha= p^*_1\), only the most likely generation ($K=1$) is selected.
In practice, we observe that the probability gap between the top response and the others can be significant, particularly when the top response is much more probable than the others. By setting $\alpha$ to avoid distortion caused by low-probability responses, this constraint ensures that only high-confidence generations are included in the uncertainty calculation. This approach is crucial for improving the reliability of the results. For a deeper understanding, including qualitative examples where our method captures uncertainty more effectively than fixed-$K$ baselines, and where probability gaps influence the retained generations, please refer to Appendix~\ref{app:examples}. The optimal value of $\alpha$ for each setting (model-dataset pair) is determined through a grid search on a small validation set comprising approximately 100 samples, with further details provided in the following section.
%Critical case: The value of \(\alpha=0\) (when we select all \(K\) generations), whereas \(\alpha=1\) (i.e. no generation is found). In the critical case there is no satisfied generation, we will use probability of the most likely generation \(p(y^*_1|x)\) as the uncertainty measure. This constraint effectively filters out low-probability generations, ensuring that only responses with high enough probabilities contribute to the final uncertainty estimation. Ideal case, \(K\)=1, we have the NLL as a baseline score [Rethinking].

% \input{sections/experiment}
\section{Experiments and Results}

\subsection{Experimental Setup}

\begin{table*}[ht]
  \centering
  \begin{tabular}{cccccccccc}
    \toprule
    \textbf{Dataset} & \textbf{Model} & \textbf{SD} & \textbf{SE} & \textbf{Deg} & \textbf{NE} & \textbf{PE} & \textbf{ALL} & \textbf{NLL} & \textbf{PRO (Ours)}\\
    % \midrule
    % \multirow{5}{*}{CoQA} 
    % & Gemma-2B      & 0.738 & 0.579 & 0.697 & 0.606 & 0.628 & 0.668 & \underline{0.771} & \textbf{0.772}       \\
    % & Gemma-7B      & \underline{0.740} & 0.552 & 0.691 & 0.532 & 0.648 & 0.668 & \textbf{0.768} & \textbf{0.768}       \\
    % & Llama2-13B    & 0.730 & 0.611 & \underline{0.747} & 0.609 & 0.666 & 0.704 & \textbf{0.806} & \textbf{0.806}       \\
    % & Falcon-11B    & \textbf{0.740} & 0.532 & 0.617 & 0.445 & 0.589 & 0.461 & 0.656 & \underline{0.700}       \\
    % & Falcon-40B    & \textbf{0.698} & 0.548 & 0.512 & 0.482 & 0.632 & 0.494 & 0.625 & \underline{0.655}       \\
    \midrule
    \multirow{5}{*}{TriviaQA} 
    & Gemma-2B      & 0.799 & 0.668 & 0.746 & 0.692 & 0.624 & 0.789 & \underline{0.806} & \textbf{0.819}       \\
    & Gemma-7B      & 0.831 & 0.690 & 0.715 & 0.702 & 0.652 & \underline{0.833} & 0.812 & \textbf{0.841}       \\
    & Llama2-13B    & \textbf{0.862} & 0.682 & \underline{0.802} & 0.551 & 0.552 & 0.624 & 0.684 & \underline{0.802}       \\
    & Falcon-11B    & \underline{0.706} & 0.592 & \textbf{0.710} & 0.555 & 0.604 & 0.577 & 0.668 & 0.668       \\
    & Falcon-40B    & 0.700 & \underline{0.724} & 0.722 & 0.674 & 0.623 & 0.658 & \textbf{0.765} & \textbf{0.765}       \\
    \midrule
    \multirow{5}{*}{SciQ}
    & Gemma-2B      & 0.719 & 0.570 & 0.725 & 0.601 & 0.605 & 0.719 & \underline{0.728} & \textbf{0.751}       \\
    & Gemma-7B      & 0.741 & 0.622 & 0.699 & 0.658 & 0.678 & \underline{0.765} & 0.755 & \textbf{0.787}       \\
    & Llama2-13B    & 0.706 & 0.574 & \textbf{0.720} & 0.481 & 0.543 & 0.515 & 0.600 & \underline{0.716}       \\
    & Falcon-11B    & 0.724 & 0.554 & 0.771 & 0.561 & 0.603 & 0.573 & \underline{0.797} & \textbf{0.799}      \\
    & Falcon-40B    & \underline{0.668} & 0.613 & 0.626 & 0.592 & 0.577 & 0.660 & \textbf{0.674} & \textbf{0.674}       \\
    \midrule
    \multirow{5}{*}{NQ}
    & Gemma-2B      & 0.618 & 0.599 & 0.620 & 0.600 & 0.613 & 0.607 & \underline{0.694} & \textbf{0.696}       \\
    & Gemma-7B      & 0.670 & 0.621 & \underline{0.691} & 0.662 & 0.566 & \textbf{0.698} & 0.683 & \underline{0.691}       \\
    & Llama2-13B    & 0.627 & 0.562 & 0.713 & 0.540 & 0.649 & 0.691 & \underline{0.737} & \textbf{0.740}      \\
    & Falcon-11B    & 0.636 & 0.591 & 0.580 & 0.515 & 0.522 & 0.512 & \underline{0.684} & \textbf{0.685}       \\
    & Falcon-40B    & 0.632 & 0.603 & 0.579 & 0.544 & 0.585 & 0.475 & \underline{0.638} & \textbf{0.645}       \\
    \midrule
    \textbf{Average AUC} & \textbf{} & {0.709} & {0.618} & {0.695} & {0.595} & {0.600} & {0.646} & {0.715} & \textbf{0.739}  \\
    \textbf{Best Count} & \textbf{} & {1} & {0} & {2} & {0} & {0} & {1} & {2} & \textbf{11}  \\
    % \textbf{Count Second Best Scores} & \textbf{} & \underline{3} & \underline{1} & \underline{3} & \underline{0} & \underline{0} & \underline{2} & \underline{8} & \underline{3}  \\
    \bottomrule

  \end{tabular}
  \caption{\label{tab:results}
    AUC performance comparison across various LLMs and datasets. The best scores for each setting are bolded, while the second-best scores are underlined. The last two rows summarize the average AUC and the total occurrences of the best scores for each setting, respectively.
  }
\end{table*}

We follow the same evaluation approach as related work by focusing on free-form question-answering tasks \cite{kuhn2023semantic, farquhar2024detecting, qiu2024semantic}.

\textbf{Datasets}. We perform experiments on three free-form question-answering datasets commonly used in the literature: TriviaQA \cite{joshi2017triviaqa}, SciQ \cite{welbl2017crowdsourcing}, and Natural Questions (NQ) \cite{kwiatkowski2019natural}. 

\textbf{Models}. We use five distinct open-source models from different families and sizes: Gemma-2B, Gemma-7B \cite{team2024gemma}, Llama2-13B \cite{touvron2023llama}, Falcon-11B, and Falcon-40B \cite{almazrouei2023falcon}.

\textbf{Baselines}. We compare our method against seven existing LLM uncertainty estimation methods: semantic density (SD) \cite{qiu2024semantic}, semantic entropy (SE) \cite{kuhn2023semantic}, degree (Deg) \cite{lin2023generating}, length-normalization predictive entropy (NE) \cite{malinin2020uncertainty}, predictive entropy (PE), average log likelihood (ALL) \cite{guerreiro2022looking}, and negative log likelihood (NLL) \cite{aichberger2024rethinking}. PE is computed using Eq.~\eqref{eq:pe_llm}.

\textbf{Evaluation metrics}. Following previous works, we evaluate uncertainty methods by measuring their ability to predict the correctness of model responses using the Area Under the Receiver Operating Characteristic curve (AUROC). Higher AUROC values indicate better performance of the uncertainty estimator, as they reflect a stronger alignment between the estimated uncertainty scores and the actual correctness of the generated answers. Specifically, we define correctness based on the most likely generation (i.e., top-1 response $y_1^*$), in line with prior studies on uncertainty estimation in LLMs \cite{kuhn2023semantic, malinin2020uncertainty, qiu2024semantic}. 

% We acknowledge the distinction, raised in prior literature, between \textit{uncertainty over the input prompt} (e.g., the model’s epistemic or aleatoric uncertainty regarding a question) and the\textit{ confidence in a particular response}. In our evaluation setting, the goal is to assess how well uncertainty estimates, derived from a model’s output distribution, can serve as proxies for the correctness of its primary response. While some refer to this as \textit{confidence scoring} or \textit{correctness estimation} \citet{lin2023generating}, we follow recent benchmarks and works (e.g., \citet{kuhn2023semantic}, \citet{guerreiro2022looking}, \citet{qiu2024semantic}) that treat it as part of uncertainty estimation. Thus, for consistency with prior evaluations, we use the term \textit{uncertainty estimation} while recognizing that the two terms may be used interchangeably in this context.

We note that some works distinguish between \textit{uncertainty estimation} over the input prompt and \textit{confidence scoring} for individual responses \cite{lin2023generating}. In our setting, we follow recent benchmarks \cite{kuhn2023semantic, guerreiro2022looking, qiu2024semantic} that evaluate uncertainty based on how well it predicts the correctness of the model’s primary output. For consistency, we adopt the term \textit{uncertainty estimation}, as used in prior work.

To determine whether a generated answer is correct, we follow standard practice by computing the F1 score of the ROUGE-L metric \cite{lin2004rouge} between the generation and the ground-truth reference. A generation is labeled as correct if this score exceeds 0.3, as in prior works. We acknowledge that automatic scoring functions such as ROUGE-L may not fully reflect human judgment. To validate the robustness of our evaluation, we provide additional results in Section~\ref{subsec:abl}, showing how performance varies across different ROUGE-L thresholds. This helps assess the sensitivity of our method and other baselines to the correctness labeling criteria.
% Although there are some limitations to using such a simple metric, it has relatively small errors in standard data sets and, therefore, remains widely used in practice. 

\textbf{Sampling}. We follow the same sampling process used in baseline method \cite{qiu2024semantic}. We use diverse beam search sampling with temperature \(\tau=1\) to generate the $N=10$ responses for each question. We utilize the responses to calculate SD, SE, NE, and PE. For Deg baseline, we adopt multinomial sampling as used in the paper \cite{lin2023generating}. ALL and NLL are computed using the probability of the most likely generation. 

While our method relies solely on output probabilities and is therefore less sensitive to semantic variation, we acknowledge that the choice of decoding strategy can influence the resulting probability distribution. Exploring alternative generation methods, such as nucleus sampling or temperature-controlled sampling, may affect uncertainty estimation and is an important direction we leave for future work. We have added this point to the Limitations section for completeness.

% Following the standard practice to assess model accuracy, we use the most likely generation (beam 0) as the target response, i.e., the response that requires an uncertainty estimation, in calculating the AUROC of uncertainty scores.

\textbf{Hyperparameter selection}. For the other methods, we adopt the hyperparameters recommended in their respective papers. For our method, PRO, we perform a grid search for the optimal $\alpha$ for each configuration using separate validation sets, comprising 100 samples. We then apply the best-found $\alpha$ to the standard test sets. We report the test results and analyze the effect of this hyperparameter ($\alpha$) in Section~\ref{subsec:bench} and Section~\ref{subsec:abl}, respectively, where we also list the final values of $\alpha$ used in each experiment.

\textbf{Computing Resources}. All experiments are conducted on a single GPU of H100 80GB.

\subsection{Benchmarking Results}\label{subsec:bench}

The results are summarized in Table~\ref{tab:results}. Our method outperforms all other baselines in 11 out of 15 cases, demonstrating its superior performance. On average, it achieves a 2.4\% improvement across the three datasets. Notably, while our method performs similarly to SD \cite{qiu2024semantic} on TriviaQA, it significantly enhances performance on SciQ and NQ, with improvements of 3.3\% and 6.5\%, respectively. These gains highlight the robustness and scalability of our approach. Furthermore, our method, which can be considered a generalized version of NLL, consistently outperforms NLL by 0.3\% on NQ and achieves a notable enhancement of 3.4\% over the TriviaQA and SciQ. This demonstrates that our method matches existing approaches' performance and introduces tangible benefits, especially in scenarios where other models struggle. As illustrated in Appendix~\ref{app:examples}, NLL may fail to detect uncertainty when the top-1 generation receives a high probability despite the presence of multiple diverse and equally plausible alternatives, leading to overconfident yet incorrect predictions.

Interestingly, the results also indicate that a larger model within the same model family does not always yield better performance in uncertainty estimation. For example, with Gemma, the smaller model (Gemma-2B) outperforms the larger one (Gemma-7B), while the reverse is true for Falcon, where Falcon-11B performs better than Falcon-40B. This suggests that model size alone does not guarantee improved uncertainty estimation, as larger models may overfit or generalize poorly on certain tasks.
This trend aligns with recent findings \citep{wen2024enhancing, lamba2025investigating}, which note that smaller models can appear to perform better because their hallucinations are easier to detect, whereas stronger models hallucinate less but more subtly, leading to lower AUROC.

\subsection{Model Analysis and Ablation Study}\label{subsec:abl}
\begin{figure*}[ht]
    \centering
    \includegraphics[width=\textwidth]{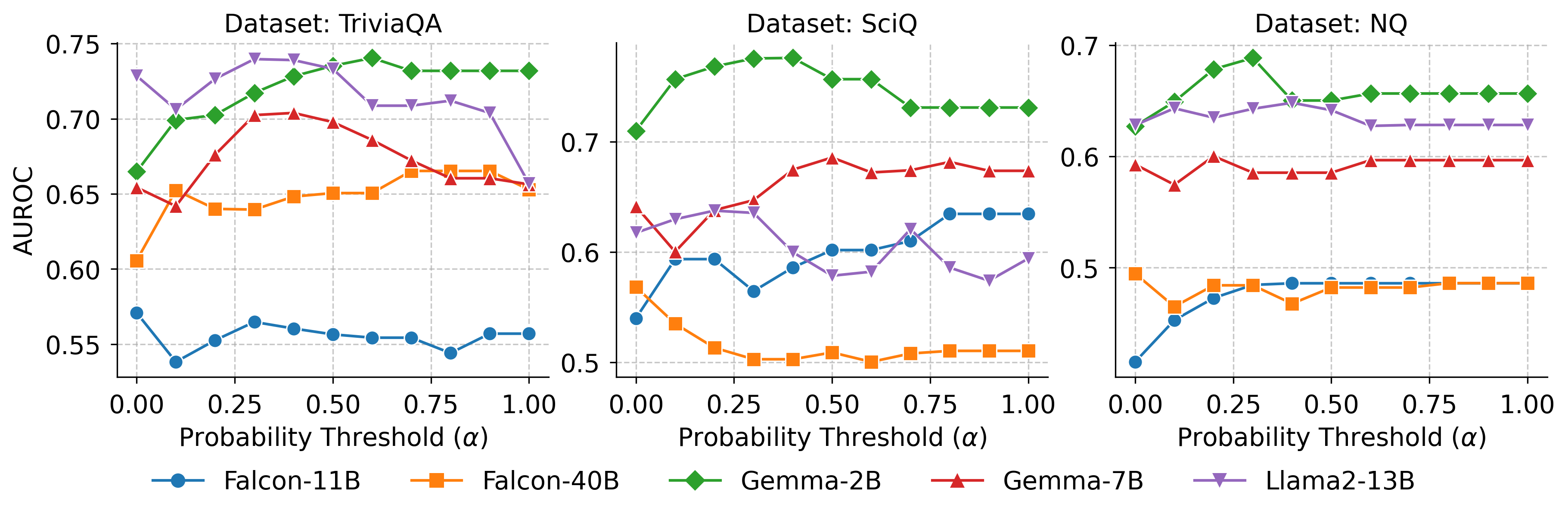}
    \caption{AUC performance when adjusting \(\alpha\) for various models and datasets.}
    \label{fig:auc_vs_alpha}
\end{figure*}

\begin{table*}[ht]
  \centering
  \begin{tabular}{lccccccccccc}
    \toprule
    % \textbf{Dataset} 
    $K$ & $1$ & $2$ & $3$ & $4$ & $5$ & $6$ & $7$ & $8$ & $9$ & $10$ & \textbf{PRO}\\
    % $K=1$ & $K=2$ & $K=3$ & $K=4$ & $K=5$ & $K=6$ & $K=7$ & $K=8$ & $K=9$ & $K=10$
    % \midrule
    % \multirow{3}{*}{CoQA} 
    % & Gemma-7B      & \textbf{0.768} & 0.716 & 0.714 & \textbf{0.768}       \\
    % & Llama2-13B    & \textbf{0.806} & 0.658 & 0.654 & \textbf{0.806}       \\
    % & Falcon-11B    & 0.656 & 0.601 & 0.639 & \textbf{0.700}       \\
    \midrule
    % \multirow{3}{*}{NQ} 
    Gemma-2B      & \underline{0.806} & 0.803 & 0.795 & 0.783 & 0.780 & 0.778 & 0.781 & 0.780 & 0.786 & 0.788 & \textbf{0.819}       \\
    Gemma-7B    & 0.812 & \underline{0.834} & 0.831 & 0.825 & 0.810 & 0.815 & 0.826 & 0.827 & 0.819 & 0.813 & \textbf{0.841}      \\
    Llama2-13B    & 0.684 & 0.795 & 0.801 & 0.800 & 0.797 & 0.799 & 0.802 & \textbf{0.809} & 0.803 & \underline{0.806} & 0.802       \\
    \bottomrule

  \end{tabular}
  \caption{\label{tab:effect-k}
    AUC performance comparison between approaches using a fixed top-\(K\) ($K=1,2,..., 10$) and PRO with adaptive constraint $\alpha$ on TriviaQA dataset. The best scores for each setting are bolded, while the second-best scores are underlined. $K=1$ represents NLL baseline.
  }
\end{table*}

\begin{figure*}[ht] 
    \centering
    \includegraphics[width=\textwidth]{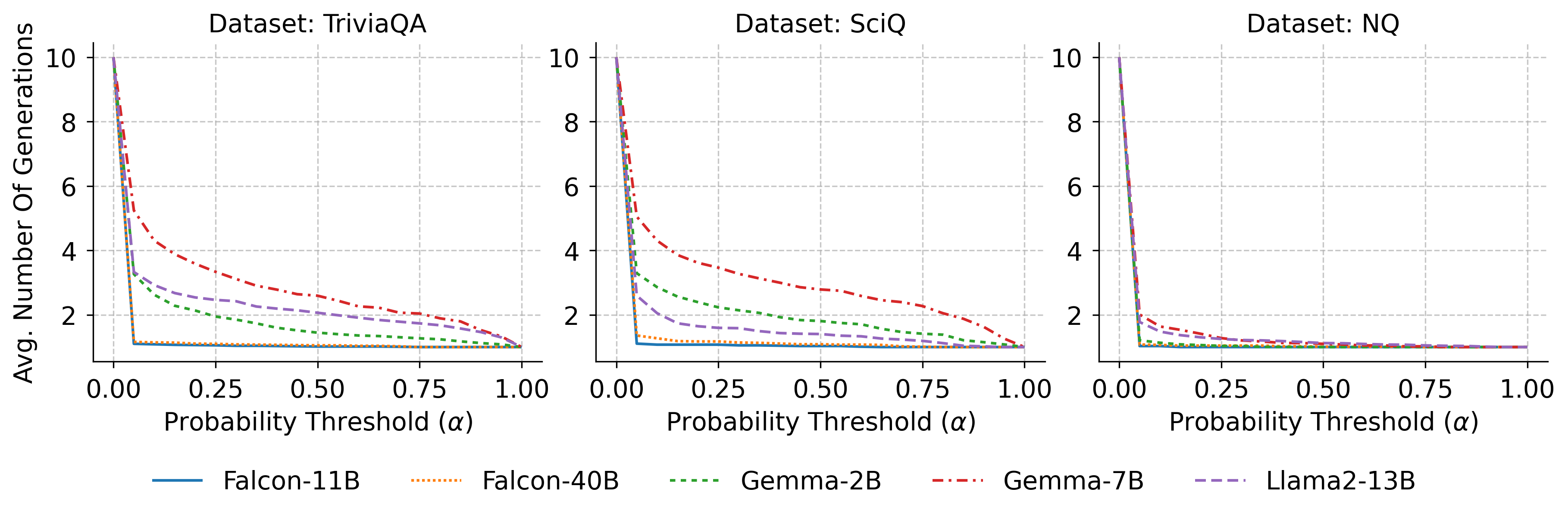}  
    \caption{Relationship between number of selected generations and \(\alpha\) across different models and datasets.}
    \label{fig:k_vs_alpha}
\end{figure*}
\textbf{Performance when adjusting the hyperparameter}

We report the change in performance when changing $\alpha$ on the validation set in Figure~\ref{fig:auc_vs_alpha}. 
% While performance differences exist across models, trends remain more consistent within the same model family (e.g., Falcon models). However, no single value of threshold works optimally across all models or datasets. 
While performance differences exist across models, the trends are model- and dataset-specific, highlighting the need for adaptive $\alpha$ selection for each setting.
Our findings suggest that the optimal \(\alpha\) varies considerably, as it is influenced not only by the model being used but also by the unique properties of each question-answer pair, including complexity, length, and contextual nuances. Note that we only tune $\alpha$ on a small validation set. The best ones are used for much bigger test sets. For example, with Gemma-7B and TriviaQA dataset, we choose $\alpha=0.3$ (see Figure~\ref{fig:auc_vs_alpha}'s first plot). The results remain robust on the test set, as shown in Table~\ref{tab:results}. The detailed $\alpha$ values for all model–dataset pairs can be found in Section~\ref{subsec:abl} for further reference.

\textbf{Effect of using more generations}
\label{subsec:multiple-k}

We investigate the impact of using a fixed top-\(K\) approach across models of similar sizes but from different families. Specifically, we compare their performance on TriviaQA dataset, where our method demonstrates superior performance over other variants of Eq. \eqref{eq:approx_final} with fixed values of $K$, including NLL ($K=1$). The results, summarized in Table~\ref{tab:effect-k}, underscore the effectiveness of applying an adaptive constraint compared to a fixed \(K\), reinforcing the importance of dynamically adjusting the number of considered generations. Using multiple generations (e.g., $K=2$) enables the method to capture a broader range of plausible responses per question, leading to more stable selection and improved performance. While fixed-$K$ methods can outperform our approach in isolated cases, they generally yield lower average performance and may lack robustness when applied across different models or datasets. In contrast, our method employs an adaptive threshold (i.e., $\alpha$) to dynamically determine the number of generations considered \textit{per question}. This adaptability allows the method to retain more diverse responses when the model exhibits uncertainty, while effectively filtering out noise when the model is confident. As illustrated in Appendix~\ref{app:examples}, examples with large probability gaps between top-ranked generations further highlight how model confidence varies, emphasizing the benefits of an adaptive approach in handling such scenarios. Additionally, we provide recommended values of $\alpha$ selected via grid search for each model and dataset in Section~\ref{subsec:abl}. When validation data is unavailable, we suggest a conservative default of $\alpha=0.4$, which might performs robustly across settings.

% Figure~\ref{fig:multiple-k} shows that increasing \(K\) allows the method to capture a broader range of plausible responses per question, contributing to more stable selection and improved results. Unlike fixed-\(K\) methods, our approach uses an adaptive threshold \(\alpha\) to select a variable number of generations \textbf{per question}. This flexibility allows the model to retain more responses when needed and ignore noise otherwise. As a result, our method often outperforms all fixed-\(K\) baselines and yields uncertainty scores that cannot be replicated by any single \(K\).

% \begin{figure}[H]
%   \centering
%   \includegraphics[width=0.95\linewidth]{AnonymousSubmission/latex/images/ablation_study_multiple_k_20250723.png}
%   \caption{AUROC for uncertainty estimation with different numbers of generations ($K$) using Gemma-7B on the SciQ dataset.}
%   \label{fig:multiple-k}
% \end{figure}
\textbf{$K$--$\alpha$ relationship}

In Figure~\ref{fig:k_vs_alpha}, we further analyze the relationship between the number of selected generations $K$ and $\alpha$ when varying $\alpha$ across models and datasets. As \(\alpha\) increases, the number of generations generally decreases. Notably, \(\alpha\) decreases more rapidly than \(k\), particularly in the first bucket of values ranging from 0 to 0.1. Moreover, the observed downward trend is also consistent across models within the same family, such as Falcon and Gemma. This reduction in selected generations can be attributed to a significant probability gap between the most likely generation and the others. In certain cases, using only the probability of the most likely generation (\(\alpha=p^*_1\)) may suffice; however, this approach generally underperforms across diverse datasets and models, lacking consistency and generalization. This is particularly evident in the TriviaQA and SciQ datasets with Falcon models, where our method produces results similar to NLL, as indicated in Table~\ref{tab:results}.
% This effect is particularly evident in the Falcon model family, where the contrast between top predictions and the others is more pronounced.

\textbf{Choice of $\alpha$}\label{app:alpha}

\begin{table}[H]
  \centering
  \begin{tabular}{cccc}
    \toprule
    \textbf{Model} & \textbf{TriviaQA} & \textbf{SciQ} & \textbf{NQ} \\
    \midrule
    Gemma-2B       & 0.40 & 0.40 & 0.70 \\
    Gemma-7B       & 0.30 & 0.50 & 0.05 \\
    Llama2-13B     & 0.10 & 0.35 & 0.30 \\
    Falcon-11B     & 0.85 & 0.80 & 0.35 \\
    Falcon-40B     & 0.65 & 0.90 & 0.45 \\
    \bottomrule
  \end{tabular}
  \caption{Values of $\alpha$ used for each setting.}
  \label{tab:alpha}
\end{table}

Table~\ref{tab:alpha} shows the values of the hyperparameter $\alpha$ selected via grid search for each combination of model and dataset. These values are determined using validation sets drawn from the training data and applied during evaluation.
While some larger or well-calibrated models (e.g., Llama2-13B) perform well with lower $\alpha$ values (e.g., 0.1-0.3), others such as Falcon-11B and Falcon-40B may require higher thresholds (e.g., 0.8-0.9). When validation data is unavailable, we recommend setting $\alpha$ between 0.3 and 0.5, with $\alpha = 0.4$ as a strong default. 

\textbf{Effect of correctness threshold}

To understand how different correctness criteria affect evaluation, we report the Area Under Curve (AUC) performance across a range of ROUGE-L F1-score thresholds from 0.1 to 0.5 (Table~\ref{tab:rouge_threshold}). These thresholds define how strictly an answer must match the reference to be counted as correct. We focus on this range as it aligns with standard practice in evaluating free-form generation, where lexical variation is expected and exact matches are not strictly required.

Across all thresholds, our method (PRO) consistently achieves the highest performance, with notable gains in the 0.3–0.5 range, which is commonly used in real-world evaluation. These results underscore PRO’s effectiveness under realistic correctness assumptions.

\begin{table}[t]
  \centering
  \begin{tabular}{ccccccc}
    \toprule
    % \textbf{Dataset} 
    \textbf{Method} & 0.1 & 0.2 & 0.3 & 0.4 & 0.5 \\
    \midrule
    SD      & 0.741     & 0.742 & 0.741 & 0.739 & 0.750         \\
    SE      & 0.628     & 0.628 & 0.622 & 0.624 & 0.636      \\
    Deg     & 0.721     & 0.723 & 0.699 & 0.712 & 0.716        \\
    NE      & 0.659     & 0.658 & 0.658 & 0.648 & 0.657       \\
    PE      & 0.649     & 0.659 & 0.678 & 0.691 & 0.708        \\
    ALL     & \underline{0.751} & \underline{0.755} & \underline{0.765} & 0.784 & 0.818       \\
    NLL     & 0.726 & 0.736 & 0.755 & \underline{0.792} & \underline{0.823}        \\
    PRO    & \textbf{0.757} & \textbf{0.763} & \textbf{0.787} & \textbf{0.802} & \textbf{0.832}       \\
    \bottomrule

    %     \textbf{Method} & 0.1 & 0.2 & 0.3 & 0.4 & 0.5 & 0.6 & 0.7 \\
    % \midrule
    % SD      & 0.741     & 0.742 & 0.741 & 0.739 & 0.750 & 0.75 & 0.713        \\
    % SE      & 0.628     & 0.628 & 0.622 & 0.624 & 0.636 & 0.634 & 0.644     \\
    % Deg     & 0.721     & 0.723 & 0.699 & 0.712 & 0.716 & 0.715 & 0.681       \\
    % NE      & 0.659     & 0.658 & 0.658 & 0.648 & 0.657 & 0.647 & 0.658      \\
    % PE      & 0.649     & 0.659 & 0.678 & 0.691 & 0.708 & 0.714 & 0.690       \\
    % ALL     & \underline{0.751} & \underline{0.755} & \underline{0.765} & 0.784 & 0.818 & 0.819 & 0.810      \\
    % NLL     & 0.726 & 0.736 & 0.755 & \underline{0.792} & \underline{0.823} & \textbf{0.841} & \textbf{0.864}       \\
    % PRO (Ours)    & \textbf{0.757} & \textbf{0.763} & \textbf{0.787} & \textbf{0.802} & \textbf{0.832} & \underline{0.830} & \underline{0.837}       \\
    % \bottomrule

  \end{tabular}
  \caption{
    AUC performance comparison under varying correctness thresholds (ROUGE-L F1 from 0.1 to 0.5) for Gemma-7B on SciQ. The best scores for each threshold are bolded, and second-best scores are underlined.}
    \label{tab:rouge_threshold}
\end{table}

\section{Conclusion}
In this work, we studied the problem of uncertainty estimation in LLMs and introduced a generalized and robust measure of predictive uncertainty based on top \(K\) generations. Our method is both simple and effective, relying exclusively on output probabilities without incorporating additional features, such as semantic similarity between responses. Our approach dynamically adjusts the number of generated samples based on probability distributions, providing greater flexibility in uncertainty estimation. These improvements strengthen the interpretability and robustness of uncertainty quantification in LLMs, contributing to more reliable and trustworthy model predictions. 

% \paragraph{Code availability} We are committed to open research and will release our implementation as open-source once the paper is published.

% \input{sections/limitation}
\section*{Limitations}

Although our approach demonstrates strong performance in uncertainty estimation, it relies on access to token logits, which may not always be available in fully black-box LLMs. As a result, our method is currently more suitable for grey-box models where token-level probability distributions can be obtained. Additionally, our approach does not consider the meaning of the question or answer, nor their relationship in measuring uncertainty. One possible direction for further research is to integrate semantic-aware features to refine uncertainty estimation and validate the approach across diverse real-world applications. Another limitation lies in our use of beam search for candidate generation. While it ensures consistent top-mode outputs and aligns with prior work, it may introduce bias due to limited generation diversity. In future work, we plan to investigate the effect of alternative decoding strategies, such as sampling-based methods, on the quality and robustness of uncertainty estimation.
% Moreover, since the optimal \(\alpha\) for each model and dataset is determined empirically, another avenue for improvement is to develop a systematic method for selecting this hyperparameter.

% \input{sections/appendix}
% \label{sec:appendix}

\bibliography{aaai2026}
% \input{reproducibility}

% \clearpage % temporarily keep this
\appendix
\setcounter{secnumdepth}{2}
\section{Appendix}\label{sec:appendix}

\subsection{Proof of Proposition \ref{prop:prop1}}
\label{sec:prop_proof}
In this section, we present a detailed proof of Proposition \ref{prop:prop1} as follows:

\begin{align}
    H(Y|x) &= -\sum_{i=1}^{\infty} p(y_i | x) \log p(y_i | x)  \label{eq:approx1} \\ 
           &= -\sum_{i=1}^{K} p^*_i \log p^*_i - \sum_{i=K+1}^{\infty} p^*_i \log p^*_i  \label{eq:approx2} \\ 
           &\geq -\sum_{i=1}^{K} p^*_i \log p^*_i - \log p^*_K \sum_{i=K+1}^{\infty} p^*_i  \label{eq:approx3} \\  
           &= -\sum_{i=1}^{K} p^*_i \log p^*_i - \log p^*_K (1 - \sum_{i=1}^{K} p^*_i)  \label{eq:approx4} \\  
           &= -\log p^*_K - \sum_{i=1}^{K} p^*_i \log \frac{p^*_i}{p^*_K}  \label{eq:approx_final_app}
\end{align}

Eq. \eqref{eq:approx3} follows directly from our definition of the top \(K\) generations. Subsequently, Eq. \eqref{eq:approx4} is derived by applying the property that the total probability sums to 1. 

\subsection{Illustrative examples}\label{app:examples}
This section provides illustrative examples demonstrating how our method more effectively estimates uncertainty in practice. These cases highlight instances where a fixed number of generations is insufficient, whereas our approach better captures confidence variations and distinguishes between likely and unlikely outcomes. Additional quantitative results supporting these observations are provided in Table~\ref{tab:examples}. 
% Note that the probabilities in the illustrative examples are estimated using Eq. \eqref{eq:nll} and may exceed 1, as they are normalized before computation.

\begin{table*}[ht]
  \centering
  \begin{tabular}{clc}
    \toprule
    \textbf{Q\&A} & \textbf{Output Sequence} & \textbf{\(p(y_i^*|x)\)} \\
    \midrule
    \midrule
    \textbf{Question}   & \textbf{Who is the sixth president of the United States?} & -        \\
    \textit{Ground Truth}    & \textit{John Quincy Adams}                           & -        \\
    \midrule
    \(y_1^*\)    & John Quincy Adams (\textbf{correct})                   & 0.455       \\
    \(y_2^*\)    & John Quincy Adams                                                & 0.455         \\
    \(y_3^*\)    & John Quincy Adams                                                & 0.455         \\
    \(y_4^*\)    & John Adams                                                       & 0.159       \\
    \(y_5^*\)    & Andrew Jackson                                                   & 0.065       \\
    \(y_6^*\)    & Andrew Jackson                                                   & 0.065       \\
    \(y_7^*\)    & Andrew Jackson                                                   & 0.065       \\
    \(y_8^*\)    & James Madison                                                    & 0.015       \\
    \(y_9^*\)    & George Washington                                                & 0.005       \\
    \(y_{10}^*\)   & John Adams, President of the United States from 1797 to 1801     & $1e^{-6}$       \\
    \midrule
    \textbf{$K=1$} & - & 0.788 \\ 
    \textbf{$K=2$} & - & 0.788 \\ 
    \textbf{$K=3$} & - & 0.788 \\ 
    \textbf{PRO ($\alpha=0.1$)} & - & \textbf{0.404} \\

    \midrule
    \textbf{Question}   & \textbf{Who plays whitey bulger's girlfriend in black mass?} & -        \\
    \textit{Ground Truth}    & \textit{actress Dakota Johnson}                           & -        \\
    \midrule
    \(y_1^*\)    & Sienna Miller (\textbf{wrong})                                & 0.144       \\
    \(y_2^*\)    & Sienna Miller                                                            & 0.144       \\
    \(y_3^*\)    & Juno Temple                                                              & 0.118       \\
    \(y_4^*\)    & Juno Temple                                                              & 0.118       \\
    \(y_5^*\)    & Dakota Johnson                                                           & 0.103       \\
    \(y_6^*\)    & Dakota Johnson                                                           & 0.103       \\
    \(y_7^*\)    & Noomi Rapace                                                             & 0.028       \\
    \(y_8^*\)    & Juliette Lewis                                                           & 0.020       \\
    \(y_9^*\)    & Julianne Nicholson                                                       & 0.019      \\
    \(y_{10}^*\) & Winona Ryder, 1990, 1991, 1992, 1993, 1994                               & $1e^{-10}$       \\
    \midrule
    K=1 & - & 1.935 \\ 
    K=2 & - & 1.935 \\ 
    K=3 & - & 2.081 \\ 
    \textbf{PRO ($\alpha=0.1$)} & - & \textbf{2.142} \\

    \midrule
    \textbf{Question}   & \textbf{Which country has the most coastline in the world?} & -        \\
    \textit{Ground Truth}    & \textit{Canada}                           & -        \\
    \midrule
    \(y_1^*\)    & Russia (\textbf{wrong})                               & 0.147       \\
    \(y_2^*\)    & Canada                                                           & 0.136        \\
    \(y_3^*\)    & Canada                                                           & 0.136        \\
    \(y_4^*\)    & Indonesia                                                        & 0.114       \\
    \(y_5^*\)    & Indonesia                                                        & 0.114       \\
    \(y_6^*\)    & Norway                                                           & 0.056       \\
    \(y_7^*\)    & Brazil                                                           & 0.044       \\
    \(y_8^*\)    & United States of America                                         & 0.007       \\
    \(y_9^*\)    & Russia, 17,                                                      & $2e^{-5}$      \\
    \(y_{10}^*\) & Australia, with 36,737 km of coastline                           & $2e^{-8}$       \\
    \midrule
    \textbf{$K=1$} & - & 1.921 \\ 
    \textbf{$K=2$} & - & 1.987 \\ 
    \textbf{$K=3$} & - & 1.987 \\ 
    \textbf{PRO ($\alpha=0.1$)} & - & \textbf{2.087} \\
        
    \bottomrule

  \end{tabular}
  \caption{\label{tab:examples}
    Examples of output sequences for NQ questions using Gemma-7B. The bolded scores indicate that our method better captures uncertainty estimation, whereas using a fixed number of generations is insufficient.
  }
\end{table*}

\subsection{Model And Data Appendix} \label{app:item_urls}
We list the links to the LLM models and datasets in Table~\ref{tab:item_urls}. 

\begin{table*}[t]
\centering
\begin{tabular}{l|l}
\midrule
\textbf{Models/Datasets} & \textbf{URL} \\
\midrule
Gemma-2B & \url{https://huggingface.co/google/gemma-2b} \\
Gemma-7B & \url{https://huggingface.co/google/gemma-7b}\\
Llama2-13B & \url{https://huggingface.co/meta-llama/Llama-2-13b-chat-hf}\\
Falcon-11B & \url{https://huggingface.co/tiiuae/falcon-11B} \\
Falcon-40B & \url{https://huggingface.co/tiiuae/falcon-40B} \\
TriviaQA & \url{https://huggingface.co/datasets/mandarjoshi/trivia_qa} \\
SciQ & \url{https://github.com/launchnlp/LitCab/tree/main/sciq} \\
NQ & \url{https://github.com/launchnlp/LitCab/tree/main/NQ} \\
\midrule
\end{tabular}
\caption{Models and Datasets Details.}
\label{tab:item_urls}
\end{table*}

\end{document}